\documentclass{article}

\usepackage[preprint,nonatbib]{neurips_2024}
\usepackage{amsmath} % assumes amsmath package installed
\usepackage{amssymb} % assumes amsmath package installed
\usepackage{mathtools}
\usepackage[hidelinks]{hyperref}
\usepackage{optidef}
\usepackage[prependcaption,colorinlistoftodos]{todonotes}
\usepackage[utf8]{inputenc} % allow utf-8 input
\usepackage[T1]{fontenc}    % use 8-bit T1 fonts
\usepackage{hyperref}       % hyperlinks
\usepackage{url}            % simple URL typesetting
\usepackage{booktabs}       % professional-quality tables
\usepackage{amsfonts}       % blackboard math symbols
\usepackage{nicefrac}       % compact symbols for 1/2, etc.
\usepackage{microtype}      % microtypography
\usepackage[ruled,linesnumbered]{algorithm2e}

\usepackage{braket}
\usepackage{tikz, pgfplots}
\usepackage{multicol}
\usetikzlibrary{patterns, positioning, calc, backgrounds}
\pgfplotsset{compat=1.17}
\usepackage{bm}
\def\BibTeX{{\rm B\kern-.05em{\sc i\kern-.025em b}\kern-.08em
    T\kern-.1667em\lower.7ex\hbox{E}\kern-.125emX}}

\newtheorem{lemma}{Lemma}
\newtheorem{theorem}{Theorem}

\newtheorem{remark}{Remark}

\newtheorem{proposition}{Proposition}
\newtheorem{proof}{Proof}

\title{Predicting AI Agent Behavior through Approximation of the Perron-Frobenius Operator}
\author{%
  Shiqi~Zhang \\
  Department of Mechanical Engineering \\
   University of California, Riverside\\
  Riverside, CA 92521 \\
  \texttt{shiqiz@ucr.edu} 
 \And 
   Darshan Gadginmath \\
   Department of Mechanical Engineering\\
   University of California, Riverside\\
  Riverside, CA 92521 \\ \texttt{dgadg001@ucr.edu} 
  \AND
   Fabio Pasqualetti
    \\
  Department of Mechanical Engineering \\
   University of California, Riverside\\
  Riverside, CA 92521 \\
  \texttt{fabiopas@ucr.edu} 
}

\begin{document}

\maketitle

\begin{abstract}

Predicting the behavior of AI-driven agents is particularly challenging without a preexisting model. In our paper, we address this by treating AI agents as nonlinear dynamical systems and adopting a probabilistic perspective to predict their statistical behavior using the Perron-Frobenius (PF) operator. We formulate the approximation of the PF operator as an entropy minimization problem, which can be solved by leveraging the Markovian property of the operator and decomposing its spectrum. Our data-driven methodology simultaneously approximates the PF operator to perform prediction of the evolution of the agents and also predicts the terminal probability density of AI agents, such as robotic systems and generative models. We demonstrate the effectiveness of our prediction model through extensive experiments on practical systems driven by AI algorithms.

% In view of this,  our  paper is dedicated to handling  the approximation and predicting  of the statistical behavior of  nonlinear dynamical systems. Exploiting  the Markovian properties and the Second Law of Thermodynamics, we conclude the basic conditions for a linear bounded operator to serve as the Perron-Frobenius operator, in light of which a simple parameterization model of Perron-Frobenius operators is developed. Parameterized by convex combination coefficient and the stationary distribution function, the proposed model characterizes the evolution tendency of the probability density functions of agents' distribution over the phase space. Fitting the Perron-Frobenius operator using the developed model  provides  us with a new approach to learn the statistical macroscopic-scale behavior of a group of  agents governed by certain nonlinear dynamical system, paving the way for prediction of the flocking behavior of the agent network. Experiments on typical and practical nonlinear dynamical systems also prove a good capability for  predicting and generalization of the proposed model.      
  
\end{abstract}

\section{Introduction}

The integration of artificial intelligence (AI) models within autonomous agents has transformed many fields, such as autonomous vehicles and personalized recommendation systems. A wide range of complex models such as LLMs, diffusion models, and different neural networks architectures have been used for the above-mentioned applications. They are trained with data that has inherent biases and can cause misaligned performance. Autonomous agents operate in dynamic environments, making decisions based on continuous feedback that allows them to learn and adapt over time. Therefore, studying the behavior and alignment of these AI-driven agents is critical for several reasons. Analyzing their actions can help prevent behaviors that conflict with human values and ethical standards. Further, understanding their behavior is essential for enhancing their efficiency and reliability, which is particularly important in safety-critical applications. These intelligent models are often complex, high-dimensional, and only partially observable over short time intervals. This complexity raises the question of what abstract properties can be efficiently quantified to delineate the boundaries of their intellectual capabilities. The design of models and understanding of the properties of AI components embedded within these agents depends crucially on the ability to analyze the interplay between AI-driven decision-making and the physical behavior of the closed-loop system. This capability is foundational for users to perceive, predict, and interact effectively with intelligent systems. It also provides the theoretical and technical basis for practical tasks such as decision-making and reinforcement learning.

Considering these challenges, it is important to develop methods that are capable of harnessing critical information to identify the behavior of AI components in closed-loop. Such an algorithm needs to provide useful laws and principles governing the behaviors of these AI-driven agents. Among the emerging methodologies, there has been a notable increase in modeling these behaviors as nonlinear dynamical systems. Originating from studies in partial differential equations (PDEs) and fluid mechanics, techniques such as Dynamic Mode Decomposition (DMD) and its generalizations have demonstrated significant capability in revealing the underlying evolutionary laws of AI agents. An alternative to the dynamic-mode-based method is the statistical perspective. This approach, drawing from statistical mechanics, models the behavior of AI agents as stochastic processes. The Perron-Frobenius (PF) operators can offer a useful tool in analyzing the evolution of stochastic processes with underlying nonlinear dynamics. Although the application of probabilistic models to learn and predict the statistical behavior of complex AI agents has increasingly attracted interest in areas such as autonomous driving, motion planning, and human-robot interaction, algorithms based on this probabilistic perspective are not yet to be fully explored.   
 % Also important will be the development of appropriate “lightweight,” easily configurable fundamental research experimental testbeds to address this problem.
 % Consider one example of a motivating problem: (1) an arbitrary set of perceptive adversarial systems or civilian/commercial systems that could be exploited by an adversary is disbursed in an operational environment with unknown or only partly known AI in their perception/action loops, (2) there is a friendly force that can actively move sensors to observe the adversarial entities, but in which the exact state and prior history of individual adversarial entities is at best only partially observable (3) the friendly force can stimulate and manipulate the adversarial systems at a meaningful level to make behavioral and modal shifts observable and to its advantage.

% AI agents capabilities are improving both physical and non-physical. Due to the complexity of the dynamics and decision policies and the lack of models, it is hard to predict the behavior of the dynamics and `alignment' of these dynamics. Biases in data leads to misalignment.  

% We provide a method to predict the behavior and alignment using finite data of the behavior of the system.

\subsection{Literature Review }
 Here we review some empirical and theoretical models from a statistical modeling perspective that have been developed with the aim of improving the prediction of the behavior of agents with AI models in the loop.   
In \cite{goswami2018constrained}, the constrained Ulam Dynamic Mode Decomposition method is presented to approximate the Perron-Frobenius operators for both the deterministic and the stochastic systems. \cite{norton2018numerical} provides a numerical approximation of the PF operators using the finite volume method. These early numerical methods lay very important foundations for the subsequent development of deep-learning-based methods, which turn out to be more scalable approaches. For example,   \cite{YM-DS-etal:2022} provides a direct learning method by training a neural network-based state transition function (operator). The works \cite{huang2019reachnn, everett2021reachability, zhang2023reachability} estimate the reachability sets of neural network-controlled systems. The works \cite{li2021physics} and \cite{de2022generic}
use deep operator networks (DeepONets) and Fourier neural operators  (FNOs) to approximate the solution trajectories of PDEs.
More recently, \cite{surasinghe2024learning} proposes an approximation method based on kernel density estimation (KDE) and    \cite{hashimoto2024deep} present a deep reproducing kernel Hilbert module (deep-RKHM), serving as a deep learning framework for kernel methods.
Beside the above-mentioned empirical methods, there are also some works that try to learn to statistical behavior from an optimal-transport-theory  perspective. For example, in  \cite{YY-LN-etal:2022}, the authors seek to recover the parameters in dynamical systems with a single smoothly varying attractor by approximating the physical invariant measure. 

A formal analysis of agents with AI models in feedback from the perspective of the evolution of the density is currently lacking in the literature. Reachability analysis needs tracking of every possible trajectory which can be computationally expensive. The density evolution approach uses a single quantity that measures the probability of evolution of trajectories and offers a significant computational advantage. Further, the density perspective allows a convenient method to verify the alignment of machine learning models. We seek to address these challenges in this work.

% Examples: robots running neural network based controllers, LLMs with complex dynamics, etc. 

% The Perron-Frobenius(PF) operator is a particularly useful tool in analyzing the statistical properties of dynamical systems. The PF operator 

% \textbf{Literature review}.
% \begin{itemize}
% \item Reachability analysis with neural networks ~\cite{ME-GH-JPH:2021, YM-DS-etal:2022}. \cite{ME-GH-JPH:2021} uses $e^{t \cdot NN}$ which makes $P \cdot f $ as a multiplication of 2 functions. 
% \item Koopman based prediction - \cite{SB-SB-etal:2021, BU-DTC-UV:2023}
% \item Optimal transport based terminal density for measure preserving systems: \cite{YY-LN-etal:2022}, ulam's method, tryphon georgiou
% \item Operator learning with DeepONets

% \end{itemize}
\subsection{Our Contributions}
% \dgcomment{
% See proposal and BAA document for introduction.
% }
Our main contributions are as follows: 
\begin{itemize}
    \item AI-driven agents behave in unpredictable ways due to machine learning black boxes. To the best of our knowledge, we are the first ones to look at this through the lens of propagation of probability densities and the PF operator. AI agents are trained with data that has inherent biases. This, coupled with the structure of machine learning models, can potentially alter the alignment of the model. To verify the alignment of the model, we predict the asymptotic behavior of the model by analyzing the terminal stationary density of the AI agents. 
    
    \item We propose PISA, a novel and scalable algorithm that can simultaneously predict the evolution of the densities of AI agents and estimate their terminal density. Our algorithm is motivated by the spectral decomposition theorem~\cite{lasota2013chaos,boyarsky1988spectral} and provides a theoretical backing for its performance. PISA simultaneously approximates the action of the Perron-Frobenius operator from the trajectory data of agents and predicts their asymptotic behavior.
    
    \item In our proposed algorithm PISA, the model complexity is indexed by the number of basis functions. The number of basis functions is a tunable parameter that can be altered according to the user's needs. We provide a theoretical guarantee of the existence of the optimal solution to our operator estimation problem. 

    \item We numerically verify the effectiveness of PISA in a variety of practical cases and compare it with existing literature. We first predict the behavior of unicycle robots driven by a controller based on diffusion models. Then we analyze the behavior of generative models from the lens of density evolution. Lastly, we apply PISA in the case of predicting the movement of pedestrians. We observe that PISA performs significantly better than the existing literature.

    % The advantage of using our model of the PF operator is that we can simultaneously predict the densities one step ahead, and also the terminal density.
\end{itemize}

% In our paper, instead of learning the statistical  behavior of the AI agents directly employing a empirical approach or formulate it as an optimal transport problem, we go further in studying the fundamental properties of Markov process. Through exploiting the \textit{spectral decomposition theorem} of Markov operators, we develop a  new time-efficient approach for learning and predicting the behavior of  Perron-Frobenius operators. 
 % $\rho^*(x)$ prediction.  With just shallow layers of NN, it performs well.

% \dgcomment{Existing approximation methods cannot guarantee a good approximation. Our algorithm which is based on spectral decomposition can find the exact PF operator with finite data.}

\section{Problem Formulation}

Consider an agent with state $x$ whose dynamics are defined by 
\begin{align}
\label{eq:exact-system}
\dot{x} = h(x,u),
\end{align}
where $u$ is an external input to the system. With a parameterized machine learning model as feedback $u = \texttt{ML}_\theta(x)$, the system's dynamics including the feedback input is given by given by 
\begin{equation}
    \label{eq:system}
    \dot{x}=h(x,\texttt{ML}_{\theta}(x)) = f(x),
\end{equation}
where $x(t)\in X \subseteq \mathbb{R}^M$ and $f(\cdot):\mathbb R^M\mapsto\mathbb R^M$ is a nonlinear continuous function. We assume that the nonlinear system \eqref{eq:system} is bounded almost everywhere, that is, $x(t)$ is in the $\mathcal{L}^{\infty}$ space. The dynamics of the probability density of the state of the system $\rho(x) \in \mathcal{L}_1$ is given by,
\begin{align}
   \label{eq:K--L} 
   \frac{\partial \rho(x,t)}{\partial t}=\sum_{m=1}^M\frac{\partial (f_m^T(x)\rho(x,t))}{\partial x_m}=A_{P}\rho(x,t).
\end{align}
Here, $A_{P}$ is called the \textit{infinitesimal generator} of the  \textit{Perron-Frobenius operator} $P$ corresponding to the nonlinear system \eqref{eq:system}, and \eqref{eq:K--L} is called the Liouville equation \cite{lasota2013chaos}.
% \dgcomment{Need to introduce what $\mathcal{L}^\infty, \mathcal{L}^1$ is. PF operator is the push-forward operator (going by Lasota and Mackay's textbook) and the Liouville operator is it's corresponding differential operator or the infinitesimal generator.}
It is sufficient that $\rho$ is any non-negative function $\mathcal{L}_1$ as it can serve as a probability density function after normalization \cite{bevanda2021koopman}.  Practically, the evolution of the system~\eqref{eq:system} can only be measured by sampling the trajectory. Given a sampling period of $\tau>0$, the discrete sampling of the system gives rise to a discretized dynamics of the system defined \eqref{eq:system}, at times $t=0,\tau,2\tau,\cdots$. For brevity, we use the notation $\rho_k(x)=\rho(x,k\tau)$. The discretized evolution of $\rho_k(x)$ is given by \begin{equation}
     \label{PF}
     \rho_{k+1}(x)=P_{\tau}\circ\rho_k(x),
 \end{equation} 
 where $P_{\tau}$ is called the Perron-Frobenius operator parameterized by the sampling period corresponding to the nonlinear dynamical system \eqref{eq:system}. We illustrate in Figure~\ref{fig:five-dim-sys} how the state trajectory $x(t)$ is coupled with the probability density $\rho(x,t)$ for the Van der Pol oscillator,
 \begin{align*}
 \dot{x}_1 = x_2, \quad
 \dot{x}_2 = \mu ( 1- x_1^2) x_2 - x_1. 
 \end{align*}
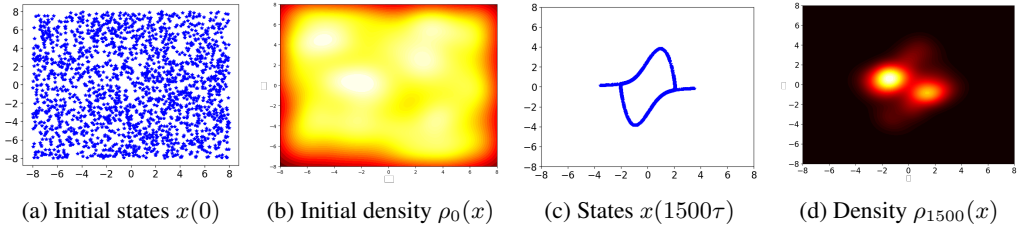
\begin{figure}[tbh]
\begin{multicols}{4}
% \hspace*{0.5cm}
\begin{tikzpicture}
  \node (img1)  {\includegraphics[width=0.2300\textwidth]{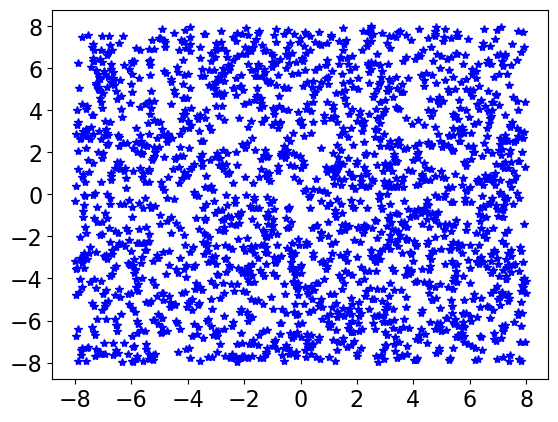}};
  % \node[above of= img1, node distance=0cm, yshift=-2.2cm,font=\color{black}]  {x};  
  % \node[below of= img1, node distance=0cm, yshift=-1.7cm,font=\color{black}]  {\small Accuracy};
  \node[above of= img1, node distance=0cm, yshift=-1.6cm,font=\color{black}]  {\small (a) Initial states $x(0)$};
  % \node[left of= img1, node distance=0cm, rotate=90, anchor=center,yshift=2.0cm,font=\color{black}] {$x_2$};
\end{tikzpicture}\columnbreak
% \hspace*{0.7cm}
\begin{tikzpicture}
  \node (img1)  {\includegraphics[width=0.2300\textwidth]{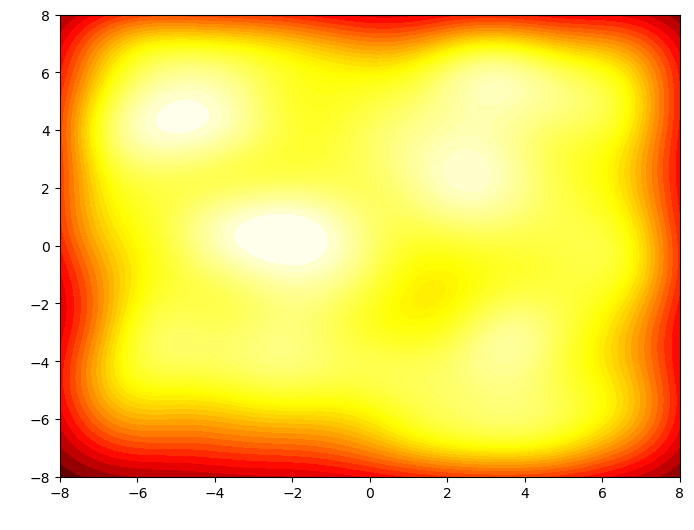}};
  % \node[above of= img1, node distance=0cm, yshift=-2.2cm,font=\color{black}]  {x}; 
  % \node[below of= img1, node distance=0cm, yshift=-1.7cm,font=\color{black}]  {\small $\log$ FPR};
  \node[above of= img1, node distance=0cm, yshift=-1.6cm,font=\color{black}]  {\small(b) Initial density $\rho_0(x)$};
  % \node[left of= img1, node distance=0cm, rotate=90, anchor=center,yshift=2.6cm,font=\color{black}] {Control effort $\|\pi(t,x)\|$};
\end{tikzpicture}\columnbreak
% \hspace*{1cm}
\begin{tikzpicture}
  \node (img1)  {\includegraphics[width=0.2300\textwidth]{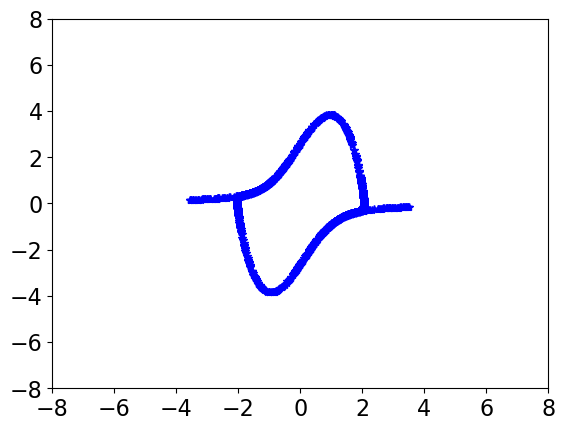}};
  % \node[above of= img1, node distance=0cm, xshift=-2cm,yshift=-1.8cm,font=\color{black}]  {x position}; 
  % \node[below of= img1, node distance=0cm, yshift=-1.7cm,font=\color{black}]  {\small $\log$ FPR};
  \node[above of= img1, node distance=0cm, yshift=-1.6cm,font=\color{black}]  {\small(c) States $x(1500\tau)$};
  % \node[left of= img1, node distance=0cm, rotate=90, anchor=center,yshift=2.6cm,font=\color{black}] {y};
\end{tikzpicture}\columnbreak
\begin{tikzpicture}
  \node (img1)  {\includegraphics[width=0.2300\textwidth]{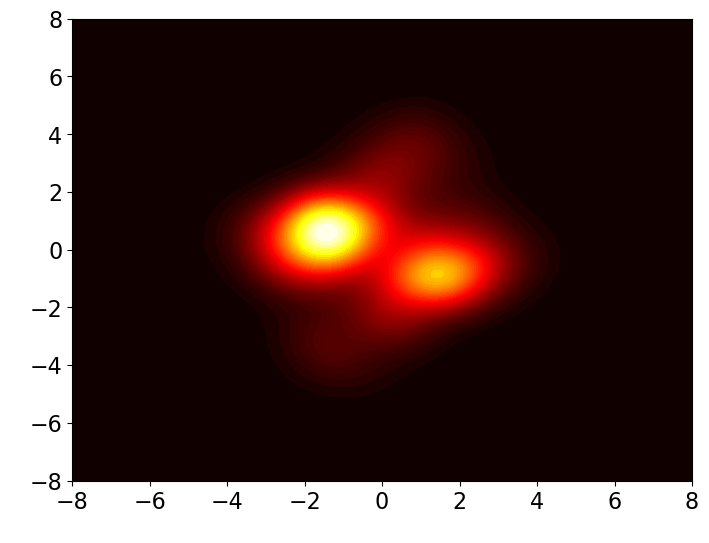}};
  % \node[above of= img1, node distance=0cm, yshift=-2.2cm,font=\color{black}]  {x}; 
  % \node[below of= img1, node distance=0cm, yshift=-1.7cm,font=\color{black}]  {\small $\log$ FPR};
  \node[above of= img1, node distance=0cm, yshift=-1.6cm,font=\color{black}]  {\small(d) Density $\rho_{1500}(x)$};
  % \node[left of= img1, node distance=0cm, rotate=90, anchor=center,yshift=2.6cm,font=\color{black}] {y};
\end{tikzpicture}
\end{multicols}
\vspace*{-2em}
\caption{Illustration of the relationship between the states and probability density of the Van der Pol oscillator in a bounded domain. The x-axis in the figures corresponds to $x_1$ and the y-axis in the figures corresponds to $x_2$. (a) Initial states of several agents driven by Van der Pol dynamics. (b) Initial density of state of agents. (c) States of the agent after time $t = 1500\tau$. (d) Density of the states at time $t = 1500\tau$. Brighter colors in (b) and (d) represent higher probability. The states are sampled with $t = 0.01$.}
\label{fig:five-dim-sys}
\end{figure}

 It is important to note that the Perron-Frobenius operator completely defines the evolution of the density of the system. Hence, our goal is to analyze the behavior of the AI-driven agents ~\eqref{eq:system}, through the estimation of the PF operator. 
 
 Our second goal is to estimate the asymptotic behavior of AI-driven agents. Several systems~\eqref{eq:system} exhibit stationary states asymptotically. For example, robotics systems are designed to stabilize certain points in the domain. Another example is a diffusion model which is trained to sample from unknown target distributions. For systems that exhibit stationary states, there exists an invariant density $\rho^*$~\cite{lasota2013chaos} for the PF operator such that 
\begin{align}
\rho^* = P_\tau \circ \rho^*.
\label{eqn:rho-star}
\end{align}
 Here agents following~\eqref{eq:system} reach $\rho^*$ asymptotically. We seek to estimate the terminal density $\rho^*$ as it provides a convenient method to assess the alignment of the AI-driven agents.

Our data consists of the state trajectory of $N$ identical agents governed by the dynamics \eqref{eq:system}. The trajectory of these agents are collected from $t=0$ to $t=T$ with a fixed sampling period  $\tau=\frac{T}{K}$. The sampled dataset is given by $\{\mathcal{X}_n\}_{n=1}^N$ of the state $x$, where $\mathcal{X}_n=[\chi_0^n,\chi_1^n,\chi_2^n, \cdots, \chi_K^n
]\in \mathbb R^{M\times (K+1)}.$ Note that the collected data set can also come from a single agent starting from $N$ different initial states. Given the data $\{\mathcal{X}_n\}_{n=1}^N$, we seek to estimate the Perron-Frobenius operator and the terminal density $\rho^*$ of the system.

% \dgcomment{Notation: Earlier we used $x$ to denote the state. Here we're using $\theta$ which is confusing for the reader. }
% \dgcomment{what is $x^*$?}
% In light of the Markovian property, we can then express the joint distribution as \begin{equation}
%     \label{eq:union}
%     p(x_0,x_1,\cdots,x_K,x^*)=p_0(x_0)p_{1|0}(x_1|x_0)\cdots p_{K|K-1}(x_K|x_{K-1})p_{\rho^*|K}(x^*|x_K).
% \end{equation}
% $$p_{k,\rho^*|K}(x_k,x^*|x_K)=p_{k|K}(x_k|x_K)p_{\rho^*|K}(x^*|x_K)$$
% $$p_{k,K,\rho^*}(x_k,x_K,x^*)=p_{k|K}(x_k|x_K)p_{\rho^*|K}(x^*|x_K)p(x_K)$$

\section{ Prediction Informed by Spectral-decomposition  Algorithm (PISA) for Learning Perron-Frobenius Operators}
We present our algorithm to estimate the PF operator in this section. Further, we predict the asymptotic behavior of the system by estimating the terminal density of the dynamical system. 

\subsection{Density Estimation using Kernel Density Estimation}

We employ Kernel Density Estimation to numerically construct the probability density $\rho_k(x)$ using the data  $\{\mathcal{X}_n\}$. We can view  $\{\mathcal{X}_n\}$ by iterating with respect to time as $\{\mathcal{Y}_k\}_{k=0}^K$, where $\mathcal{Y}_k=[\chi_k^1,\chi_k^2,\cdots,\chi_k^N]$
denotes the state vectors of $N$ particles at time $t=k\tau$. Using kernel density estimation \cite{hastie2009elements}, we then get an empirical probability distribution estimation $\rho_k(x)$. In this paper, we choose to use the Gaussian kernel for the estimation of $\rho_k$ which is given by
\begin{align}    
    \label{eq:density_0}
    \rho_k(x)= \frac{1}{N\sqrt{\det(2\pi\sigma_k^2I_M)}}\sum_{n=1}^N e^{-\frac{\left\|x-\chi_k^n\right\|^2}{2\sigma_k^2}}.
\end{align}

It is important to note that any choice of density estimation algorithm can be used with the data $\{\mathcal{X}_n\}$ to obtain $\{\rho_k(x)\}_{k=0}^K$. KDE provides a convenient choice for measuring the probability $\rho(x)$ at fixed reference points. The choice of reference points and the parameter $\sigma_k$ can be chosen by the user to better approximate $\rho_k$. {In this paper, we choose to uniformly sample the reference points in the domain $X$ to estimate every $\rho_k(x)$.} We fix a constant $\sigma_K = \sigma$ for simplicity. More specific KDE-related tools can be used based on the system in consideration and the domain, as enlisted in \cite{chen2017tutorial}. 
\subsection{Proposed Algorithm}
% We now propose an algorithm to approximate the action of the PF operator. 
We approximate the PF operator using the following model. 
\begin{align}
    \label{eq:NN_PF}
    \begin{aligned}
    P_{\tau}\circ\rho(x)
    &=\rho(x)-\sum_{i=1}^l\left(\frac{1}{l}-\texttt{A}^i_{\theta}(\rho)\right)\texttt{G}^i_{\gamma}(x).
    \end{aligned}
\end{align}

Here, we are decomposing the action of the PF operator on the density $\rho(x)$ into $2l$ components as given by $l$ functions $\texttt{A}^i_\theta(\rho)$ and $l$ functions $\texttt{G}^i_\gamma(x)$. The functions $\texttt{A}^i_\theta(\rho)$ and $\texttt{G}^i_\gamma(x)$ are parameterized by $\theta$ and $\gamma$ respectively. This method of decomposing the PF operator is guided by the spectral decomposition theorem which we elaborate on in Section~\ref{sec:SDT}. 

 % Let $\texttt{G}_{\gamma}(x)=[\texttt{G}_{\gamma}^1(x),\texttt{G}_{\gamma}^2(x),\cdots,\texttt{G}_{\gamma}^l(x)]^T$ be a vector of normalized non-negative  functions and $\texttt{A}_{\theta}(\rho)=[\texttt{A}_{\theta}^1(\rho),\texttt{A}_{\theta}^2(\rho),\cdots, \texttt{A}_{\theta}^l(\rho)]^T$ be a vector of non-negative functionals of the density  $\rho(x)$. 

 Given the decomposition of the PF operator, we propose the following loss function, guided by the spectral decomposition theorem, to learn the parameter $\theta$ and $\gamma$.
\begin{align}
\begin{aligned}
    L({\theta}, \gamma)=&\sum_{k=0}^{K-1}D\left(\rho_k(x)-\sum_{i=1}^l\left(\dfrac{1}{l}-\texttt{A}^i_{\theta}(\rho_k)\right)\texttt{G}^i_{\gamma}(x)\Bigg\Vert\rho_{k+1}(x)\right)+\lambda\sum_{i\neq j}^l\braket{\texttt{G}^i_{\gamma}(x),\texttt{G}^j_{\gamma}(x)}\\&+\mu\sum_{r=1}^lD\left(\texttt{G}^r_{\gamma}(x)-\sum_{i=1}^l\left(\dfrac{1}{l}-\texttt{A}^i_{\theta}(\texttt{G}^r_{\gamma})\right)\texttt{G}^i_{\gamma}(x)\Bigg\Vert \texttt{G}^{r+1}_{\gamma}(x)\right).
\end{aligned}
\end{align}
 We then construct PISA as the following alternating optimization algorithm to compute $\theta$ and $\gamma$, in which we choose $\texttt{A}^i_{\theta}(\rho)$ and $\texttt{G}^i_{\gamma}(x)$ to be  outputs of two distinct neural networks parameterized by $\theta$ and $\gamma$, respectively.
\begin{algorithm}
\label{algo:PF-optimization}
\SetAlgoLined
\caption{\bf{:  Prediction Informed by Spectral-decomposition  Algorithm (PISA) }}
% \begin{algorithmic}
\KwData{$l > 0$, $\lambda>0$, $\mu>0$; $\rho_k(x)$, for $k=0,1,\cdots,K$; initial values of  ${\gamma}$ and ${\theta}$; two small positive thresholds $\epsilon_1$ and $\epsilon_2$;}
\KwResult{$\gamma$ and $\theta$;}
% \State $y \gets 1$
% \State $X \gets x$
$N_\text{epochs} \gets 1000$\\
\While{$N_\text{epochs}\neq 0$}{

Solve the following optimization problem to get $\gamma^*$
\begin{align}\label{opt}
 \begin{split}
 &\min_{\gamma} ~~L(\theta,\gamma)\\
 &\text{s.t.}~ ~ \texttt{G}^i_{\gamma}(x)\geq0 \text{ and } \int \texttt{G}^i_{\gamma}(x)dx =1, \text { for } i=1,\cdots,l;
 \end{split}
 \end{align}
 
\noindent\If{$\left\|\gamma^*-\gamma\right\|\geq \epsilon_1$} 
{$\gamma \gets \gamma^* $ }
Solve the following optimization problem to get $\theta^*$
\begin{align}\label{opt22}
 \begin{split}
 &\min_{\theta} ~~L(\theta,\gamma)\\
  &\text{s.t.} ~~\texttt{A}^i_{\theta}(\rho)\geq0, \text { for } i=1,\cdots,l;
 \end{split}
 \end{align}

\noindent\If{$\left\|\theta^*-\theta\right\|\geq \epsilon_2$} 
{$\theta \gets \theta^* $} 
$N_\text{epochs} = N_\text{epochs} -1 $}
\end{algorithm}

An important aspect of PISA is that it can also predict the terminal density of the PF operator. The estimate of terminal density of $P_{\tau}$ can be expressed as \begin{equation}
    \label{rhostar}
    \rho^*(x)=\frac{1}{l}\sum_{i=1}^l        \texttt{G}^i_{\gamma}(x). \end{equation}

% \dgcomment{Add remark saying that we use one neural network with $l$ outputs instead of $l$ different neural networks. Using $l$ independent neural networks increases model complexity and aligns with the spectral decomposition theorem exactly but we can compensate for the errors by using correlation with one neural network that has $l$ outputs. Our model encompasses the case with $l$ different A and G. Addresses model complexity and generality. }

\section{Numerical Experiments}

We present the effectiveness of PISA on different numerical testbeds. We performed the numerical experiments on a machine with Intel i9-9900K CPU with 128GB RAM and the Nvidia Quadro RTX 4000 GPU. In our numerical experiments, we compare the performance of PISA with the that of ~\cite{YM-DS-etal:2022}. Particularly, \cite{YM-DS-etal:2022} approximates the PF operator as 
\begin{align*}
\rho_{k+1} = e^{t \cdot \texttt{NN}_\delta(x,t)} \rho_k.
\end{align*}
Here, note that $\texttt{NN}_\delta$ approximates the Liouville operator $A_P$ given in \eqref{eq:K--L}. Then $e^{t \cdot \texttt{NN}_\delta}$ is an approximately linear solution to~\eqref{PF}. 

\subsection{ Unicycle Model with an NN-Based Controller}
We first consider agents following the unicycle dynamics,
\begin{align}
\dot{x}_1 = u_1 \cos(x_3), \quad \dot{x}_2 = u_1 \sin(x_3), \quad \dot{x}_3 = u_2.
\end{align}
We consider the task of controlling $N$ unicycle agents using a diffusion-model-based controller developed in \cite{KE-DG-FP:2024}. The task is to start from a uniform distribution and finally reach two Gaussian distributions centered at $-4\mathbf{1}$ and  $4\mathbf{1}$, respectively. With the number of agents as $N=1000$, we generated trajectories of length $T = 8$ seconds with a sampling period of $\tau = 0.01$. The initial and final states of the data are represented in Figure~\ref{fig:unicycle}(a). We predicted the terminal density $\rho^*$ as given by~\eqref{rhostar}.  The estimated terminal density corresponding to the targets the agents are trained to reach is depicted in Figure~\ref{fig:unicycle}(b). 
In Figure~\ref{fig:unicycle}(c), we compare PISA against \cite{YM-DS-etal:2022}, using the same amount of training samples and the same size of neural networks for both algorithms. We set $l=5$ for PISA and used fully connected feed-forward neural networks with 3 hidden layers to learn $\texttt{A}_\theta$, $\texttt{G}_\gamma$ and $\texttt{NN}_\delta$.The two algorithms are trained for 1000 iterations and tested on the testing dataset of length $T=3$ seconds. The y-axis of Figure~\ref{fig:unicycle}(c) depicts the KL divergence between the predicted density and the true density in the testing dataset. Initially, it is evident that \cite{YM-DS-etal:2022} performs better than PISA due to the nature of the model. Particularly, \cite{YM-DS-etal:2022} chooses a linear solution to the PDE~\eqref{PF} which explains the better performance initially but results in a rapid decrease in performance. We see that PISA performs better by one order of magnitude than \cite{YM-DS-etal:2022} over a long time horizon. 
\begin{figure}[tbh]
\begin{multicols}{3}
% \hspace*{0.8cm}
\begin{tikzpicture}
  \node (img1)  {\includegraphics[width=0.2800\textwidth]{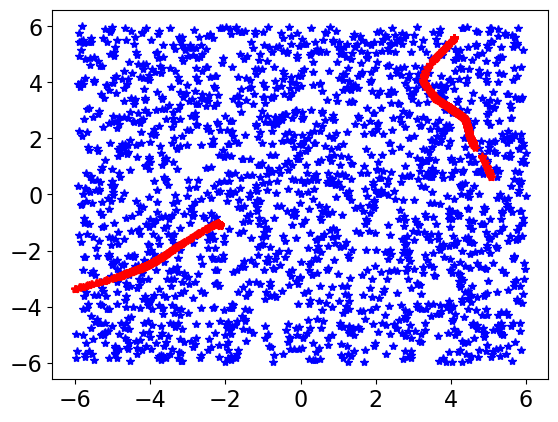}};
  \node[left of= img1, node distance=0cm, xshift=-2.10cm,font=\color{black}]  {\small $x_2$};
  \node[below of= img1, node distance=0cm, yshift=-1.70cm,font=\color{black}]  {\small $x_1$};
  \node[above of= img1, node distance=0cm, yshift=-2.2cm,font=\color{black}]  {\small (a) Dataset };
  % \node[left of= img1, node distance=0cm, rotate=90, anchor=center,yshift=2.0cm,font=\color{black}] {$x_2$};
\end{tikzpicture}\columnbreak
% \hspace*{0.7cm}
\begin{tikzpicture}
  \node (img1)  {\includegraphics[width=0.2800\textwidth]{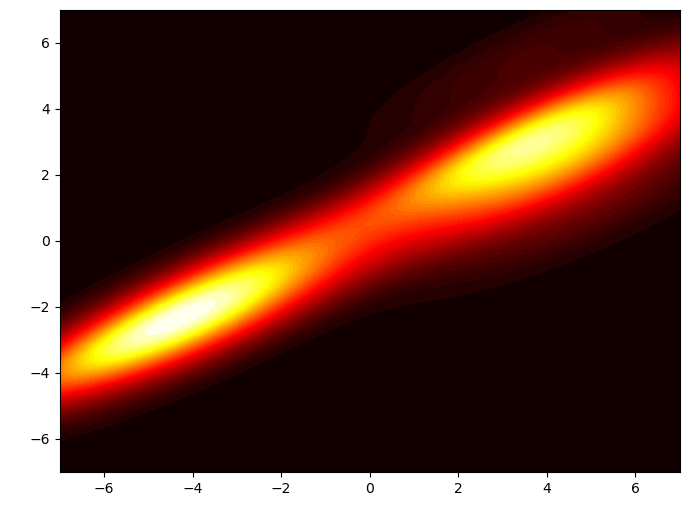}};
  \node[left of= img1, node distance=0cm, xshift=-2.10cm,font=\color{black}]  {\small $x_2$};
  \node[below of= img1, node distance=0cm, yshift=-1.70cm,font=\color{black}]  {\small $x_1$};
  \node[above of= img1, node distance=0cm, yshift=-2.2cm,font=\color{black}]  {\small (b) Predicted $\rho^*$ };
\end{tikzpicture}\columnbreak
% \hspace*{1cm}
\begin{tikzpicture}
  \node(img1)  {\includegraphics[width=0.2500\textwidth]{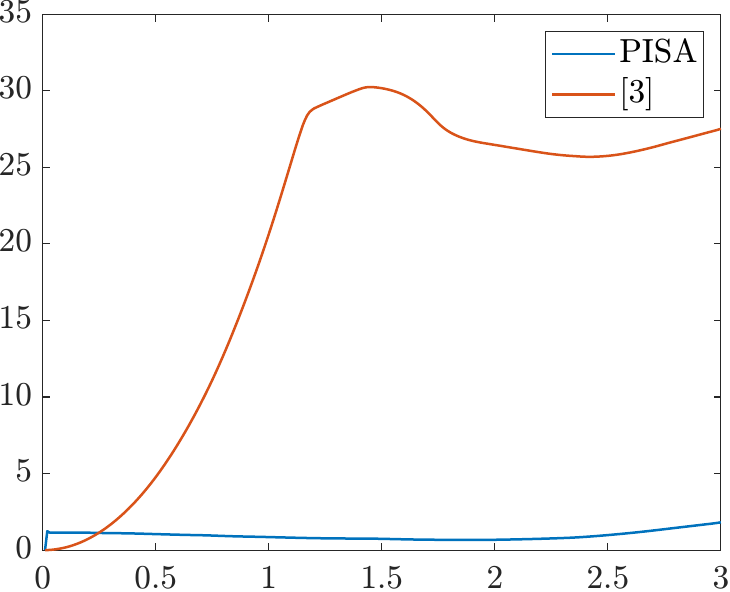}};
  \node[left of= img1, rotate = 90, node distance=0cm, yshift=2.2cm,font=\color{black}]  {\small $D(\hat{\rho}_t\|\rho_t)$};
  \node[below of= img1, node distance=0cm, yshift=-1.70cm,font=\color{black}]  {\small time};
  \node[above of= img1, node distance=0cm, yshift=-2.2cm,font=\color{black}]  {\small(c) Comparison with \cite{YM-DS-etal:2022}};
  % \node[left of= img1, node distance=0cm, rotate=90, anchor=center,yshift=2.6cm,font=\color{black}] {y};
\end{tikzpicture}
\end{multicols}
\vspace*{-2em}
\caption{Experiments on the unicycle model data generated using controller from~\cite{KE-DG-FP:2024}. (a) Initial states are uniformly distributed, depicted in blue, final states are in red. (b) The estimated $\rho^*$ by PISA. (c) Performance comparison between PISA and \cite{YM-DS-etal:2022} on testing data of length $3$ seconds. Performance metric is the KL divergence between predicted and true densities. Our dataset contains $K= 800$ sampled instants, wherein  we use the first $500$ samples of each trajectory as the training data set and the remaining $300$ as the testing set.}
\label{fig:unicycle}
\end{figure}
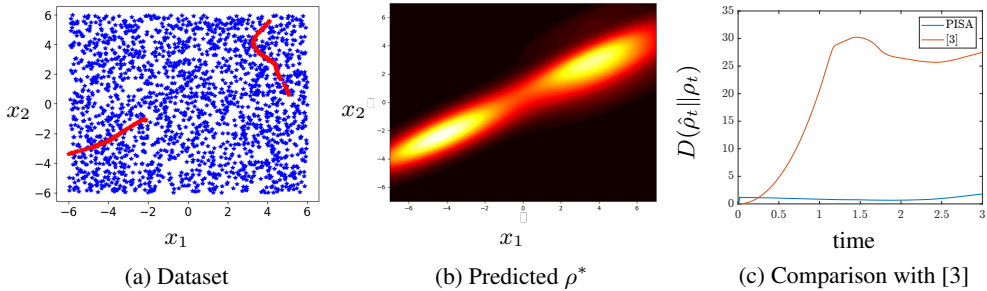

% Note that $K= 800$ in this experiment. We use the first $500$ samples of each trajectory as the training data set and the remaining $300$ as the testing set. 

\subsection{Predicting Behavior of Score-Based Generative Model}

We consider the problem of analysis of the behavior of generative models. Generative models behave as evolving agents by taking data samples from an initial noise density and reaching an unknown target density. Generative models are complex and their behavior can be analyzed through the trajectory of converting noise samples to data samples. Our task is to analyze the behavior of generative models from the lens of evolving probability densities using their sampling trajectories. We particularly consider the case of diffusion models based on estimating the score~\cite{YS-JSD-etal:2020}. In the diffusion model we use the data lies in five dimensions. In score-based generative models, new data samples from unknown target distributions are generated using a bi-directional scheme. Given data from an unknown target distribution, noise is sequentially added to data samples using a Stochastic Differential Equation in the forward process until a desired noise distribution is reached. To sample new data from the target distribution, this forward process is reversed so that a sample starts from the noise distribution and ends as a sample from the target distribution. Particularly,~\cite{YS-JSD-etal:2020} reverses the forward process by learning the \emph{score} of the data distribution. We seek to study the behavior of these score-based generative models by their action on samples in the reverse process. We show that we could not only predict the behavior of diffusion models but also potentially identify the target distribution by estimating $\rho^*$. 

We consider the following ODE for the score-based generative model,
\begin{align}
\dot{x} = u
\end{align}
Here, $x,u \in \mathbb{R}^5$. In the forward process $u~\sim\mathcal{N}(\bm{0},I)$, whereas in the reverse process, $u$ denotes the output of a neural network $\texttt{S}_\psi(x,t)$ that estimates the score. 
The neural network $\texttt{S}_\psi(x,t)$ with five outputs approximates the score of the forward process. In this experiment, the task is to sample from Gaussians centered at $-4\bm{1}$ and $3\bm{1}$. The noise distribution is the uniform distribution over the domain $[-8,8]^5$. In the forward process, we start with $N=3000$ samples from $\mathcal{N}_1(-4\bm{1},0.2I)$ and  $\mathcal{N}_1(3\bm{1},0.2I)$. The data samples are diffused in the forward process for a time period of 6 seconds. In the reverse process, to sample from the desired distributions, we learn the score as proposed in \cite{YS-JSD-etal:2020}. Once the score is sufficiently learned using a neural network, we record $N=3000$ trajectories in the reverse process for a time period of 6 seconds. The first four seconds of the dataset constitute our training dataset and the last 2 seconds constitute the testing dataset. 

It is evident from Figure~\ref{fig:5d-score}(a) that PISA makes an accurate estimation of the target distribution. We can see that the estimated $\rho^*$ corresponds to that of the aforementioned target distribution. This implies that the diffusion model is well aligned with its intended task. Further, in Figure~\ref{fig:5d-score}(b), we compare PISA with~\cite{YM-DS-etal:2022} in predicting the evolution of $\rho_k$. We choose $l=10$ for PISA and feedforward neural networks with 3 hidden layers for $\texttt{A}^i_\theta$, $\texttt{G}^i_\gamma$ and $\texttt{NN}^i_\delta$. We use the KL divergence between the predicted $\hat{\rho}_k$ and the true $\rho_k$ from the testing dataset. Once again, we see that~\cite{YM-DS-etal:2022} performs better initially due to its linear solution to the PDE but its performance deteriorates rapidly. PISA retains relatively much better performance over a longer time horizon.

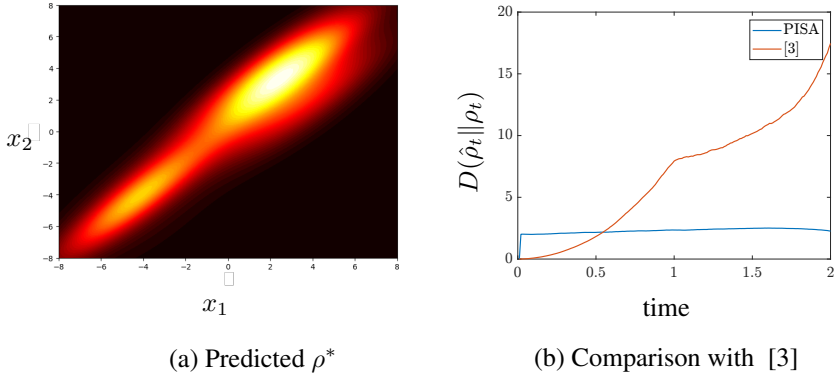
\begin{figure}[tbh]
\begin{multicols}{2}
\hspace*{0.8cm}
\begin{tikzpicture}
  \node (img1)  {\includegraphics[width=0.360\textwidth]{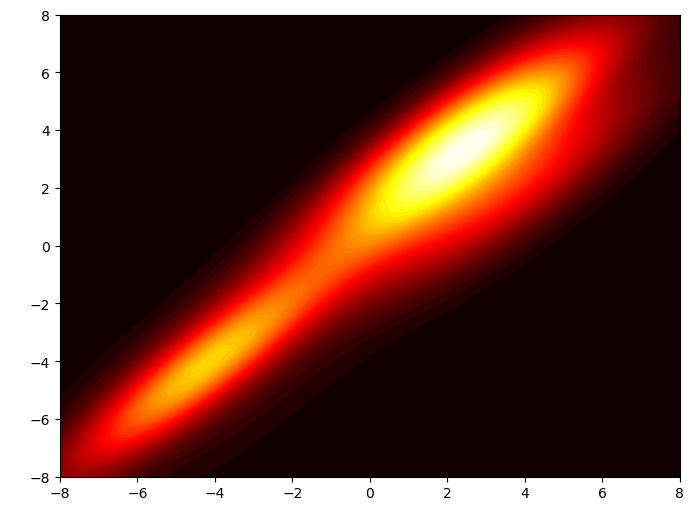}};
  \node[left of= img1, node distance=0cm, xshift=-2.60cm,font=\color{black}]  {$x_2$};
  \node[below of= img1, node distance=0cm, xshift = 0.0cm,yshift=-2.20cm,font=\color{black}]  {$x_1$};
  \node[above of= img1, node distance=0cm, xshift = 0.5cm, yshift=-2.9cm,font=\color{black}]  { (a) Predicted $\rho^*$};
  % \node[left of= img1, node distance=0cm, rotate=90, anchor=center,yshift=2.0cm,font=\color{black}] {$x_2$};
\end{tikzpicture}\columnbreak
\hspace*{0.4cm}
\begin{tikzpicture}
  \node (img1)  {\includegraphics[width=0.3200\textwidth]{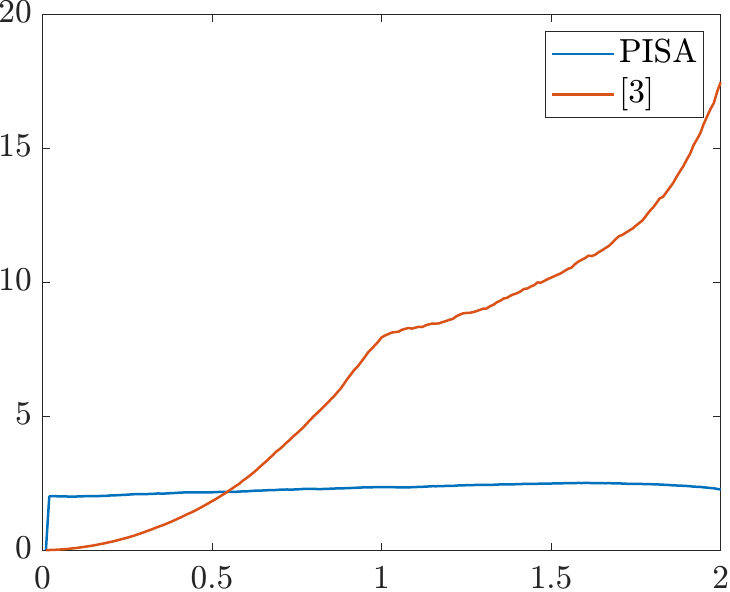}};
  \node[left of= img1, rotate = 90, node distance=0cm, yshift=2.60cm,font=\color{black}]  {$D(\hat{\rho}_t\|\rho_t)$};
  \node[below of= img1, node distance=0cm, yshift=-2.20cm,font=\color{black}]  {time};
  \node[above of= img1, node distance=0cm, yshift=-2.9cm,font=\color{black}]  { (b) Comparison with ~\cite{YM-DS-etal:2022} };
\end{tikzpicture}
\end{multicols}
\vspace*{-2em}
\caption{Experiments on the five dimensional score-based generative model~\cite{YS-JSD-etal:2020}. (a) Predicted terminal density $\rho^*$ projected on the first two dimension. Brighter colours indicate a higher probability. (b) Comparison of performance of PISA and \cite{YM-DS-etal:2022} on testing dataset. }
\label{fig:5d-score}
\end{figure}

\subsection{UCY Pedestrian Dataset}

We show the effectiveness of PISA on a physical data by applying it to the UCY pedestrian dataset~\cite{lerner2007crowds}. Here, the task is to predict the movement of pedestrians by estimating the evolution of the density of pedestrians. The dataset consists of videos of pedestrians walking in several regions as depicted in Figure~\ref{fig:pedestrian}(a). 
 We use the Zara01 subsection of the dataset in our experiments. We obtained pre-processed data from the code repository of~\cite{salzmann2020trajectron++}, where the video was processed to obtain the $x$ and $y$ coordinates of the position of the pedestrians. Here we consider each pederstrian as an agent in our analysis. We assume that every pedestrian is identical and their movement is governed by the dynamics given in~\eqref{eq:system}. Given the positions of pedestrians as depicted in Figure~\ref{fig:pedestrian}(a), we approximate the probability density of the pedestrians as depicted in Figure~\ref{fig:pedestrian}(b).  
 
 In Figure~\ref{fig:pedestrian}(c), we once again compare PISA with the exponential model~\cite{YM-DS-etal:2022} on the test data for the first 200 time samples. We choose $l=5$ and feedforward neural networks with 3 hidden layers for $\texttt{A}^i_\theta$, $\texttt{G}^i_\gamma$ and $\texttt{NN}^i_\delta$. Here, we see that the initial time period in which the exponential model works better than PISA is significantly shorter due to the model inaccuracy. However, PISA continues to perform well over a longer time horizon. It is also important to note that both models have significantly higher estimation errors in the testing performance for this experiment compared to performance in experiments on the unicycle model and the score-based generative model. This is due to the stochasticity of the data and the assumption we make about the nature of the pedestrians. Further, with new pedestrians entering the scene and existing pedestrians leaving the scene, there are jumps in the probability density which further reduces the performance of KDE-based methods. Better density approximation algorithms that are suited for stochastic data and for incorporating jumps in the probability density can be employed to obtain an improvement in the performance.
\begin{figure}[tbh]
\begin{multicols}{3}
% \hspace*{0.8cm}
\begin{tikzpicture}
  \node (img1)  {\includegraphics[width=0.2700\textwidth]{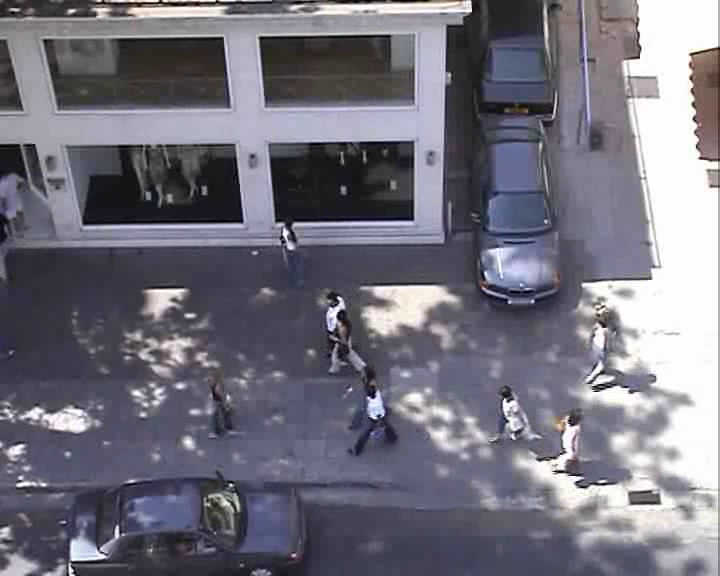}};
  % \node[left of= img1, node distance=0cm, xshift=-2.10cm,font=\color{black}]  {\small $x_2$};
  % \node[below of= img1, node distance=0cm, yshift=-1.70cm,font=\color{black}]  {\small $x_1$};
  \node[above of= img1, node distance=0cm, yshift=-2.2cm,font=\color{black}]  {\small (a) Snapshot from dataset };
  % \node[left of= img1, node distance=0cm, rotate=90, anchor=center,yshift=2.0cm,font=\color{black}] {$x_2$};
\end{tikzpicture}\columnbreak
% \hspace*{0.7cm}
\begin{tikzpicture}
  \node (img1)  {\includegraphics[width=0.3600\textwidth]{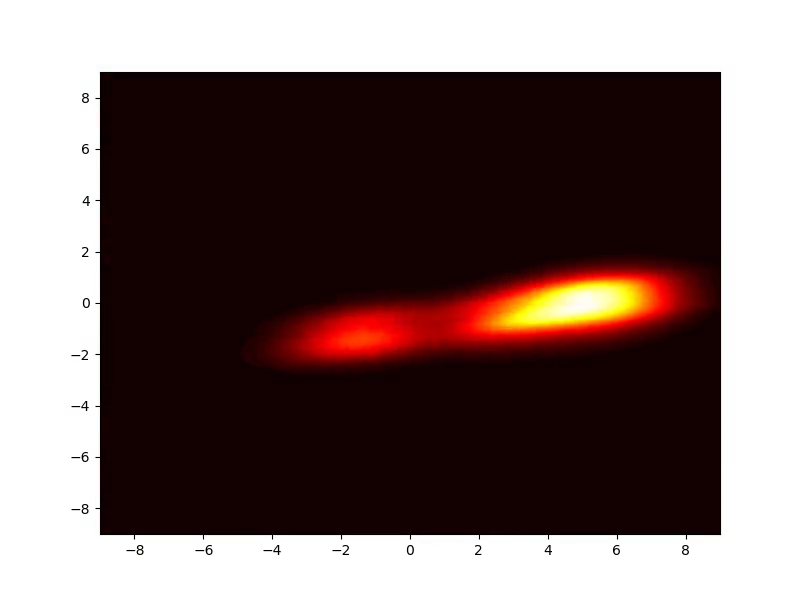}};
  \node[left of= img1, node distance=0cm, xshift=-2.30cm,font=\color{black}]  {\small $x_2$};
  \node[below of= img1, node distance=0cm, yshift=-1.90cm,font=\color{black}]  {\small $x_1$};
  \node[above of= img1, node distance=0cm, yshift=-2.2cm,font=\color{black}]  {\small (b) Estimated density $\rho_k$ };
\end{tikzpicture}\columnbreak
\hspace*{-0.4cm}
\begin{tikzpicture}
  \node(img1)  {\includegraphics[width=0.2700\textwidth]{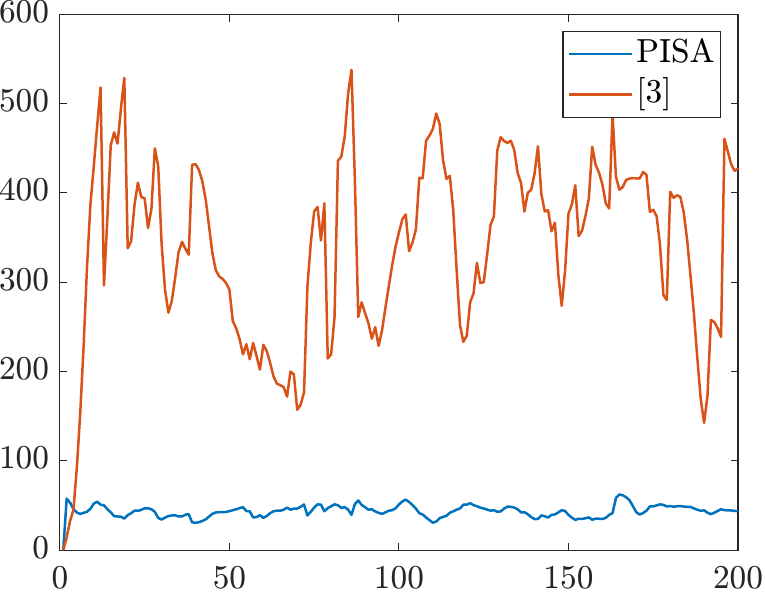}};
  \node[left of= img1, rotate = 90, node distance=0cm, yshift=2.2cm,font=\color{black}]  {\small $D(\hat{\rho}_t\|\rho_t)$};
  \node[below of= img1, node distance=0cm, yshift=-1.70cm,font=\color{black}]  {\small time};
  \node[above of= img1, node distance=0cm, yshift=-2.2cm,font=\color{black}]  {\small(c) Comparison with \cite{YM-DS-etal:2022}};
  % \node[left of= img1, node distance=0cm, rotate=90, anchor=center,yshift=2.6cm,font=\color{black}] {y};
\end{tikzpicture}
\end{multicols}
\vspace*{-2em}
\caption{Experiments on the UCY pedestrian dataset. (a) A snapshot from the dataset. (b) Corresponding estimated probability density. (c) Comparison between PISA and \cite{YM-DS-etal:2022} in the estimation of future probability densities.  }
\label{fig:pedestrian}
\end{figure}
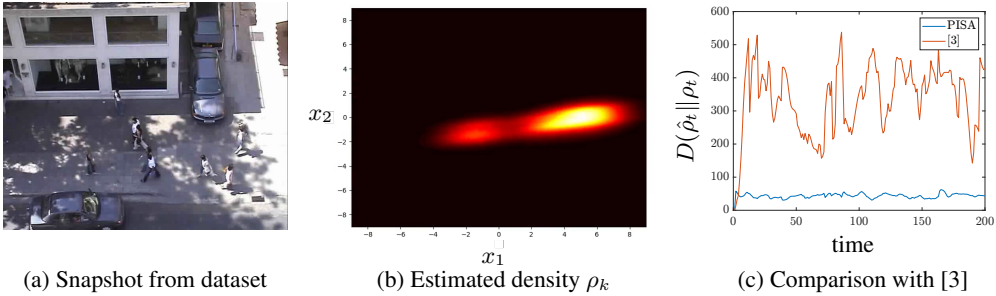

% \begin{itemize}
% \item Study with different L values for the unicycle data to compare estimation of $\rho^*$.
% \item Calculation of time and model complexity for solving (16) (See carefully : $O(n \log n)$?). Highlight the fact that we can use small neural networks to predict $\rho_{k+t}$ and $\rho^*$.
% \end{itemize}

\section{Theoretical Foundations of Spectral Decomposition}
\label{sec:SDT}
The Perron-Frobenius operator  $P_{\tau}$ is a constrictive Markov operator~\cite{lasota2013chaos} that pushes forward the probability distribution $\rho_k(x)$ to a stationary distribution $\rho^*(x)$ corresponding to the attractors of dynamical systems. This evolution of the probability distribution is in fact a \textit{Markov Process}. For Markov operators with  the constrictive property, we have the following spectral decomposition theorem.

\begin{lemma}
\cite{lasota2013chaos, boyarsky1988spectral} \label{thm:spec_decom}
     Let $P$ be a constrictive Markov operator. Then there exists an integer $l$, two sequences of non-negative functions $g_i(x)\in\mathcal{L}_1$ and $h_i(x)\in\mathcal{L}_{\infty}$, $i=1,2,\cdots,l,$ and an  operator $Q:\mathcal{L}_1\mapsto\mathcal{L}_1$ such that for all $\rho(x)\in\mathcal{L}_1$, $P\circ \rho(x)$ can be written in the form \begin{equation}
        \label{eq:spectral}
    P\circ\rho(x)=\sum_{i=1}^la_i(\rho)g_i(x)+Q\circ \rho(x),\end{equation}
    where $$a_i(\rho)=\int \rho(x)h_i(x)dx.$$
    The functions $g_i(x) $ and the operator $Q$ have the following properites:
    
    1) Each $g_i(x)$ is normalized to one and \begin{equation}
         \label{eq:ortho}g_i(x)g_j(x)=0, \quad\text{ for all } i\neq j,   
        \end{equation} i.e., the density functions $g_i(x)$ have disjoint supports;
        
        2)  For each integer $i$ there exists a unique integer $\alpha(i)$ such that \begin{equation}
    \label{eq:permute} 
        P\circ g_i(x)=g_{\alpha(i)}(x).\end{equation} where $\alpha(i)\neq\alpha(j)$ for $i\neq j$. Thus, $P$ just permutes the functions $g_i(x)$; 
        
        3) Moreover, \begin{equation}
\label{eq:decay}\|P^nQ\circ\rho(x)\|\rightarrow0
        \end{equation} as $n\rightarrow\infty$ for every $\rho(x)\in \mathcal{L}_1$.
\end{lemma}

Lemma~\ref{thm:spec_decom} states that the action of the PF operator can be decomposed into $l$ components through the functionals $a_i(\rho)$ and the functions $g_i(x)$. Here $l$ is a finite integer that serves as a measure of the model complexity of PISA. Further, the operator $Q$ captures the effect of the terminal density on $\rho(x)$. As $t\rightarrow\infty$, the action of $Q$ on $\rho(x)$ decays to 0. This drives our motivation to use $Q \circ \rho(x)$ as
\begin{align}\label{eqn:Qrho}
Q\circ \rho(x) = \rho(x) - \rho^*(x).
\end{align}
Further, as $t\rightarrow\infty$, we can see that 
\begin{align}\label{eqn:rhostar-g}
\rho^*(x) = \frac{1}{l}\sum_i^l g_i(x).
\end{align}
This implies that the density functions $g_i$ serve as a basis for the stationary terminal density $\rho^*$. It is easy to verify for~\eqref{eqn:rhostar-g} that $P\circ\rho^*(x) = \rho^*(x)$ through the permutation property. Given the Lemma~\ref{thm:spec_decom}, \eqref{eqn:Qrho}, and \eqref{eqn:rhostar-g}, we provide a sufficient condition on the output of our algorithm PISA.

\begin{theorem} For systems~\eqref{eq:system} that have a stationary terminal density, there exists a finite $l$, an operator $Q$, $l$ non-negative functionals $\texttt{A}^i_{\theta}(\rho)$ and $l$  densities $\texttt{G}_{\gamma}(x)$ such that the loss $L(\theta,\gamma) = 0$.
\end{theorem}
\begin{proof}
We provide a brief overview of the proof. Lemma~\ref{thm:spec_decom} guarantees the existence of $l$ functions $a_i(\rho)$ and $g_i(x)$ that exactly decompose the action of the PF operator. These are approximated using neural networks $\texttt{A}^i_{\theta}(\rho)$ and $\texttt{G}^i_{\gamma}(x)$, respectively. The cost function $L$ is designed to satisfy the properties of $a_i$ and $g_i$. The first term in the cost function $L$ addresses the propagation of the PF operator. The second term addresses the orthogonality property of every $g_i$ and the last term captures the permutative property $g_i$. \hfill $\blacksquare$
\end{proof}

% This theorem serves as a core foundation for us to develop the algorithm for learning the Perron-Frobenius operators, which we shall see in the next section.

% \begin{equation}\label{opt}
%  \begin{array}{rrclcl}
%  \displaystyle \min_{g_i(x)} & \multicolumn{2}{l}{L(\theta_i,g_i(x))}\\
%  \text{s.t.} & g_i(x)\geq0 \text{ and } \int g_i(x)dx =1;
%      \\
%  \end{array}
%  \end{equation}
% \begin{equation}\label{opt2}
%  \begin{array}{rrclcl}
%  \displaystyle \min_{\theta_i} & \multicolumn{2}{l}{L(\theta_i,g_i(x))}
%  \end{array}
%  \end{equation}

\begin{remark}
    It should be noticed that the spectral decomposition theorem guarantees the existence of a decomposition in the form of  \eqref{eq:spectral}, it yet does not specify the form of $Q$ and $g_i(x)$, as well as the number $l$. Also, there is no guarantee that the decomposition is unique. Therefore, our choice that $Q\circ\rho(x)=\rho(x)-\rho^*(x)$ is just one of many possibilities. An alternative for the action of $Q$ is $Q\circ\rho(x) = D(\rho\|\rho^*)$, which is also a decreasing quantity. This can be verified by the laws of thermodynamics which guarantee decreasing of  relative entropy. Empirically, we have seen that $Q\circ\rho(x)=\rho(x)-\rho^*(x)$ serves as a better choice for $Q$. Moreover, Lemma~\ref{thm:spec_decom} guarantees a finite model complexity through the number of basis functions $l$. This is a tunable parameter in our algorithm that can adjusted to achieve a desired accuracy of prediction. As described in our numerical experiments, it turns out that PISA makes good approximations of both \eqref{eq:NN_PF} and \eqref{rhostar}. Our relatively small neural networks to approximate $a_i(\rho)$ and $g_i(x)$ through $\texttt{A}^i_\theta$ and $\texttt{G}^i_\gamma$, respectively, are empirically efficient choices.  
    %It is also worthy noting that the model complexity necessitated by certain prediction accuracy, indexed by the number $l$, is increasing in a logarithmic rate with respect to the dimension of the variables, that is $l\simeq O(\log M)$. These observations will be shown in detail in the next section.  
\end{remark}

\section{Conclusion}

In this study, we explore the task of predicting the behavior of agents controlled by AI models, using a probabilistic framework to analyze the evolution of probability densities. Our proposed algorithm, PISA, effectively estimates the Perron-Frobenius (PF) operator that characterizes the evolution of these densities, thus enabling predictions of both the short and long-term behavior of AI-driven agents. This ability to forecast asymptotic behaviors is critical for assessing whether such agents align with human values and requirements. Currently, our approach utilizes kernel density estimation to approximate the probability densities from individual trajectories; however, this method introduces potential inaccuracies. Optimizing the choice of kernel functions and adopting more sophisticated density estimation techniques could significantly enhance our model's performance. Additionally, a comprehensive evaluation of the model's complexity and computational efficiency remains to be conducted, which will be crucial for practical applications and further scalability.

\newpage

\bibliographystyle{unsrt}
\bibliography{refs}

\newpage

\section{Supplemental Material} 
\subsection{Preliminaries and Mathematical Notations}\label{sec:pre}
KL divergence between two probability distributions $f(x)$ and $g(x)$ is defined as $$D(f\|g)=\int f(x)\log\frac{f(x)}{g(x)}dx,$$
which serves as a measure of the distance between the two distributions. Notice also that the $D(f\|g)=0$ if and only if $f(x)-g(x)=0$ almost everywhere. Moreover, we denote  $\mathbb R^M$ the vector space of all $M$ dimensional real vectors. The inner product of two vectors (functions) $f(x)$ and $g(x)$ in Hilbert space is defined as $$\braket{f(x),g(x)}=\int f'(x)g(x)dx,$$ where  $f'(x)$ denotes the conjugate transpose of $f(x)$. Also, we use $\|\cdot\|$ to denote the norm of a function (vector) in Banach space. The $\mathcal{L}_2$ norm of a vector $f(x)$ in Hilbert space is defined as $$\|f\|_2=\sqrt{\int f'(x)f(x)dx}.$$ 
The $\mathcal{L}_1$ norm of a vector $f(x)$ in Banach space is defined as $$\|f\|_1={\int |f(x)|dx}.$$
The $\mathcal{L}_{\infty}$ norm of a vector $f(x)$ in Banach space is defined as $$\|f\|_{\infty}=\sup_{x}{ |f(x)|dx}.$$
We also denote $\mathcal{L}_1$ the set of all absolutely integrable functions and $\mathcal{L}_{\infty}$ the set of all  functions which are almost bounded everywhere.
\subsection{Proposition of Perron-Frobenius Operators }\label{Prop_Markov}

\begin{proposition}\label{prop:PF}
    \cite{lasota2013chaos} Suppose we have a nonlinear bounded dynamical system \eqref{eq:system} and the corresponding Perron-Frobenius  operator $P_{\tau}$. Then it follows that \begin{itemize} 
        \item $P_{\tau}$ is a linear operator;
        \item $P_{\tau}\circ\rho(x)$ is non-negative if $\rho(x)$ is non-negative;
         \item Integral invariance: $$\int P_{\tau}\circ\rho(x)dx=\int\rho(x)dx;$$
         \item $P_{\tau}$ has a fixed point (probability distribution function) such that $$P_{\tau}\circ\rho^*(x)=\rho^*(x)$$
         and $$\lim_{n\rightarrow\infty}P_{\tau}^n\circ\rho(x)=\rho^*(x)$$
         which is also denoted  in some literature that $P_{\tau}$ has a preserved measure $\mu_{\rho^*}$.
    \end{itemize}
\end{proposition}

\newpage

\end{document}